\documentclass[acmsmall,nonacm,screen]{acmart}

\usepackage{tikz}
\usepackage{multirow}
\usepackage{tabularx}

\definecolor{vizblue}{HTML}{2A78D6}
\definecolor{vizorange}{HTML}{EB6834}
\definecolor{inkgray}{HTML}{52514E}

\newcommand{\pending}[1]{\textcolor{red}{[#1]}}
\newcommand{\E}{\mathbb{E}}
\newcommand{\beff}{\beta_{\mathrm{eff}}}
\newcommand{\sig}{\textsuperscript{\dag}}
\newcommand{\unadj}{\textsuperscript{*}}

\begin{document}

\title{Do Artificial Neural Networks Need Dendritic Nanotubes?}
\subtitle{Two Reasons for Shrinking Importance Estimates in Continual Learning}

\author{Saurav K. Aryal}
\email{saurav.k.aryal@gmail.com}
\affiliation{%
  \institution{Howard University}
  \department{Institute for Human Centered Artificial Intelligence}
  \city{Washington}
  \state{DC}
  \country{USA}}
\renewcommand{\shortauthors}{S. K. Aryal}

\begin{abstract}
Regularization-based continual learning protects each parameter in proportion to an estimate of its importance, yet these estimates come from finite samples and heuristics and are used as if they were exact. We report four studies that start from a biologically inspired mechanism and end with a statistical explanation of it. Motivated by dendritic nanotubes, we let transient lateral links diffuse importance tags between hidden units. Links between co-active units gave no significant gain, but random links with lifetimes of 25 or 200 steps raised final accuracy on 20-task Permuted MNIST by 1.5 to 2.6 points. We show that random diffusion implements linear shrinkage of importance toward the layer mean in expectation. A simulation of the link process alone predicts shrinkage weights of 0.23 and 0.52 for these lifetimes, and graph-free shrinkage at 0.25 and 0.5 reproduces both the accuracy and the dispersion of the consolidated importance, while every activity-dependent rule did worse than random links. A five-seed screen across EWC, SI, and MAS found nominal gains for SI and MAS on Permuted MNIST and losses on task-incremental Split CIFAR-100, where forgetting was small and shrinkage cost plasticity. Varying the number of importance samples then separates two mechanisms on Permuted MNIST. For EWC, the gain fell by about ten points per tenfold increase in samples, from 19.6 points with 10 samples to $-0.4$ with 1000, a variance effect. For MAS, a gain of about one point persisted at 1000 samples, consistent with a bias in the importance measure rather than noise. Shrinkage had no detectable effect for MAS on a domain-incremental CIFAR-100 benchmark. Finally, we show that a standard heteroscedastic empirical Bayes rule shrinks the largest Fisher-type estimates most, which may explain its poor results, and point to a log-scale alternative.
\end{abstract}

\ccsdesc[500]{Computing methodologies~Lifelong machine learning}
\ccsdesc[300]{Computing methodologies~Neural networks}
\ccsdesc[300]{Computing methodologies~Regularization}

\keywords{continual learning, catastrophic forgetting, elastic weight consolidation, synaptic intelligence, memory aware synapses, shrinkage estimation, empirical Bayes, Fisher information, dendritic nanotubes}

\maketitle

\section{Introduction}\label{sec-intro}

Neural networks trained on a sequence of tasks tend to overwrite what they learned before, a failure known as catastrophic forgetting \citep{mccloskey1989catastrophic,goodfellow2013empirical}. Regularization-based methods address it by anchoring each parameter near its value after earlier tasks, with a strength proportional to an estimate of that parameter's importance. Elastic weight consolidation (EWC) uses the diagonal of the Fisher information \citep{kirkpatrick2017overcoming}, synaptic intelligence (SI) accumulates each parameter's contribution to the decrease of the loss along the training path \citep{zenke2017continual}, and memory aware synapses (MAS) measure how strongly the network output depends on each parameter \citep{aljundi2018memory}. In all three, the importance vector is the method's entire account of what should be preserved.

These vectors are estimates, and noisy ones. EWC is usually computed from a modest number of single-example gradients, SI from correlated steps of one optimization trajectory, and MAS from output gradients on a sample of inputs. How the Fisher information is computed changes the performance of EWC substantially \citep{vandeven2025computation}, and the empirical Fisher differs systematically from the true one \citep{kunstner2019limitations,martens2020new}. Yet the methods treat the estimates as exact. Statistics has a standard response to estimating many quantities at once. When three or more normal means are estimated together, shrinking the individual estimates by a suitable amount toward a common value lowers the expected total squared error \citep{stein1956inadmissibility,james1961estimation}, and empirical Bayes methods choose the amount of shrinkage from the data \citep{efron1975data}. An importance vector poses a problem of this kind, with hundreds of thousands of means. The simplest shrinkage operator is familiar under another name. Pulling every importance value toward its group mean by a fixed fraction adds to the method's own penalty an L2 penalty toward the anchor that is uniform within each group, and L2 penalties of this kind are a standard baseline \citep{kirkpatrick2017overcoming}. To our knowledge, however, nobody has studied systematically when and why shrinking importance estimates helps in regularization-based continual learning.

We did not start from statistics. This work began with a biological observation. \citet{chang2025intercellular} identified dendritic nanotubes, thin tubes that connect the dendrites of neighboring neurons, in mammalian cortex. In dissociated neurons these actin-rich tubes propagated calcium over long distances, and they actively transported small molecules and amyloid-$\beta$. Calcium signals gate synaptic plasticity \citep{malenka2004ltp}, and synaptic tagging and capture describes how a local tag decides which synaptic changes become permanent \citep{frey1997synaptic,redondo2011making}. A channel that moves plasticity signals between neurons is therefore a natural candidate for a cross-cellular consolidation mechanism. We implemented the idea as transient lateral links that diffuse importance tags between the hidden units of a layer and tested it on Permuted MNIST. The mechanism failed in the form we expected and helped in a form we did not. Links formed between co-active units gave no significant gain, while links formed at random reliably improved retention. This paper reports how that observation led to shrinkage, and what shrinkage does for standard importance estimators.

We make four contributions.
\begin{itemize}
\item We report a negative result for two biologically motivated mechanisms, dense lateral connections pruned during training and transient links between co-active units, and a positive result for transient random links on 20-task Permuted MNIST (Studies 1 and 2).
\item We show that random links implement linear shrinkage of importance toward the layer mean in expectation, with a weight set by the link lifetime and formation rate (Proposition~\ref{prop-links}). A simulation of the link process alone predicts the weights, and graph-free shrinkage at those weights reproduces both the accuracy and the dispersion of the consolidated importance.
\item We screen fixed and empirical Bayes shrinkage across EWC, SI, and MAS on Permuted MNIST and Split CIFAR-100 to see where shrinkage may help and where it costs plasticity (Study 3).
\item We separate two reasons why shrinkage helps by varying the number of samples used to estimate importance. On Permuted MNIST, the benefit for EWC is a variance effect that vanishes with enough samples, while the benefit for MAS has a component that persists, which we read as bias (Study 4). We also show that a standard heteroscedastic empirical Bayes rule shrinks the largest Fisher-type estimates most (Proposition~\ref{prop-eb}), which may explain its poor results.
\end{itemize}

The experiments use small networks and standard benchmarks, and they ran on a single laptop. We therefore present the work as a careful empirical note built around a simple mechanism, not as a new state of the art. Section~\ref{sec-background} reviews related work, Section~\ref{sec-methods} defines the estimators, the shrinkage operators, and two propositions, Section~\ref{sec-design} describes the four studies, Section~\ref{sec-results} reports the results, and Section~\ref{sec-discussion} discusses their implications and limitations.

\section{Background and Related Work}\label{sec-background}

\subsection{Quadratic consolidation penalties}
Regularization-based methods add a penalty that keeps parameters close to their values after earlier tasks \citep{delange2022continual,wang2024comprehensive}. With anchor $\theta^{*}$ and importance $\Omega$, the penalty is $\lambda\sum_i \Omega_i(\theta_i-\theta^{*}_i)^2$. EWC derives $\Omega$ from a Laplace approximation to the posterior over earlier tasks and uses the diagonal Fisher information \citep{kirkpatrick2017overcoming}. Online variants accumulate importance over tasks and keep a single anchor \citep{schwarz2018progress,huszar2018note}, as we do. SI and MAS replace the Fisher with path-integral and output-sensitivity measures \citep{zenke2017continual,aljundi2018memory}, and \citet{chaudhry2018riemannian} combine Fisher and path-integral terms. Several methods assign importance to units rather than to single weights \citep{ahn2019uncertainty,jung2020continual}, which is the granularity of our Studies 1 and 2. Bayesian approaches make the uncertainty of each parameter explicit, through structured Laplace approximations \citep{ritter2018online}, variational inference \citep{nguyen2018variational,ebrahimi2020uncertainty}, or metaplasticity that scales learning by synaptic uncertainty \citep{bonnet2025bayesian}. \citet{benzing2022unifying} showed that the importance measures of SI and MAS are closely related to the square root of the Fisher information, a relation we return to in Section~\ref{sec-discussion}.

\subsection{Estimating importance}
Each importance measure is an estimate, and the estimators differ in how much noise and bias they carry. The diagonal Fisher can be computed with labels sampled from the model, which gives the true Fisher, or with the observed labels, which gives the empirical Fisher, and the two can differ sharply \citep{kunstner2019limitations,martens2020new}. \citet{vandeven2025computation} showed that how EWC computes the Fisher, including the number of samples, has a large effect on its performance. Importance estimates also become less reliable when data are noisy or shift between tasks. In EEG-based emotion recognition, \citet{peire2026affect} found that importance heuristics degraded under noise and covariate shift and that importance accumulated over tasks over-constrained the model. None of these works shrinks the estimates.

\subsection{Shrinkage and empirical Bayes}
When three or more independent normal means with a common known variance are estimated together, the vector of sample means is inadmissible under total squared error. Shrinking it by a suitable amount toward a fixed point lowers the expected total error for every value of the means \citep{stein1956inadmissibility,james1961estimation}, and the same holds for shrinkage toward the average of the estimates when there are four or more. Empirical Bayes reads such estimators as posterior means under a prior whose spread is estimated from the data, which sets the shrinkage weight automatically \citep{efron1975data}. Shrinkage of the same kind regularizes large covariance matrices \citep{ledoit2004well}. In our setting the means are the importance values of individual parameters or units, and the common value is their mean within a layer or a parameter tensor. Importance estimates are neither normal nor equally noisy, so these results motivate shrinkage rather than guarantee that it helps.

\subsection{Biological motivation}
Synaptic tagging and capture proposes that a transient local tag marks the synapses whose changes may become permanent \citep{frey1997synaptic,redondo2011making}. Models of synapses with internal states on several timescales, and of metaplasticity, show how such signals can support continual learning \citep{benna2016computational,kaplanis2018continual,laborieux2021synaptic}, and synaptic plasticity has been described as Bayesian inference over weights \citep{aitchison2021synaptic}. Tunneling nanotubes transport material between cells \citep{rustom2004nanotubular}, and \citet{chang2025intercellular} identified dendrite-to-dendrite nanotubes in mammalian cortex that propagated calcium in dissociated neurons and transported small molecules and amyloid-$\beta$. Our transient links are a coarse engineering abstraction of such a channel, and nothing in this paper tests nanotube physiology.

\section{Methods}\label{sec-methods}

\subsection{Setting}
A network with parameters $\theta$ learns tasks $t=1,\dots,T$ in sequence. After task $t$ it stores an anchor $\theta^{(t)}$ and adds that task's normalized importance to a running total,
\begin{equation}\label{eq-accum}
\Omega^{(t)} = \Omega^{(t-1)} + \frac{\tilde{\omega}^{(t)}}{\operatorname{mean}\bigl(\tilde{\omega}^{(t)}\bigr)}, \qquad \Omega^{(0)}=0,
\end{equation}
where $\tilde{\omega}^{(t)}$ is the importance estimate for task $t$, possibly shrunk, and the mean runs over all protected parameters (Studies 3 and 4) or over the units of a layer (Studies 1 and 2). While learning task $t$, the network minimizes
\begin{equation}\label{eq-penalty}
\mathcal{L}_t(\theta) + \lambda \sum_i \Omega^{(t-1)}_i \bigl(\theta_i - \theta^{(t-1)}_i\bigr)^2 .
\end{equation}
The normalization in Eq.~(\ref{eq-accum}) gives every task the same protection budget. A shrinkage operator therefore redistributes protection among parameters without changing its total, and $\lambda$ has the same meaning in every arm.

\subsection{Importance estimators}
\paragraph{Unit-level squared gradients (Studies 1 and 2)} For unit $u$ of a layer, with incoming weights $W_{u\cdot}$ and bias $b_u$, the importance at training step $s$ is $\iota_u(s) = \operatorname{mean}_j (\partial\ell_s/\partial W_{uj})^2 + (\partial\ell_s/\partial b_u)^2$, computed from the minibatch gradient of the task loss $\ell_s$. The task estimate sums $\iota_u(s)$ over the steps of the task. The penalty anchors all incoming weights and the bias of unit $u$ with the same strength $\lambda\Omega_u$. The output layer is consolidated in the same way, but links and shrinkage act only on the two hidden layers.

\paragraph{EWC} The diagonal Fisher information is estimated from $n$ single-example gradients with labels sampled from the model's predictive distribution, $\hat{\omega}_i = \frac{1}{n}\sum_{k=1}^{n} \bigl(\partial \log p_\theta(\tilde{y}_k \mid x_k)/\partial\theta_i\bigr)^2$ with $\tilde{y}_k \sim p_\theta(\cdot \mid x_k)$, where the inputs $x_k$ are drawn from the training data of the task after training on it. A unit-level variant, EWC-unit, averages the same contributions over each unit's incoming weights and bias.

\paragraph{MAS} The importance is the mean output sensitivity, $\hat{\omega}_i = \frac{1}{n}\sum_{k=1}^{n} \bigl|\partial \lVert f_\theta(x_k)\rVert_2^2/\partial\theta_i\bigr|$, where $f_\theta$ returns the logits.

\paragraph{SI} The importance is the path integral of the loss decrease, normalized by the squared total change, $\hat{\omega}_i = \max\bigl(0, \sum_s -g_i(s)\,\Delta\theta_i(s)\bigr) / \bigl((\Delta\theta_i)^2 + \xi\bigr)$, where $g(s)$ is the gradient of the task loss before the penalty is added, $\Delta\theta(s)$ is the parameter change at step $s$, $\Delta\theta$ is the change over the whole task, and $\xi = 0.1$.

For EWC and MAS the sampling variance of each estimate is $v_i = s_i^2/n$, with $s_i^2$ the sample variance of the $n$ per-example contributions. For SI we compute the estimate as the mean of the per-step contributions $-g_i(s)\,\Delta\theta_i(s)$, clipped at zero and divided by $(\Delta\theta_i)^2+\xi$, which differs from the formula above only by a constant factor that the normalization in Eq.~(\ref{eq-accum}) removes. Its $v_i$ is the squared standard error of that mean divided by $\bigl((\Delta\theta_i)^2+\xi\bigr)^2$. The per-step contributions are autocorrelated, so $v_i$ is only approximate for SI.

\subsection{Shrinkage operators}
Let $G$ be a group of estimates, either a parameter tensor or the units of a layer for unit-level methods, with mean $\bar{m}_G$. We compare a fixed rule and an empirical Bayes (EB) rule,
\begin{align}
\tilde{m}_i &= (1-\beta)\, m_i + \beta\, \bar{m}_G, \label{eq-fixed}\\
\tilde{m}_i &= (1-B_i)\, m_i + B_i\, \bar{m}_G, \qquad B_i = \frac{v_i}{v_i + \hat{\tau}^2}, \qquad \hat{\tau}^2 = \max\bigl(0,\, \operatorname{Var}_G(m) - \operatorname{mean}_G(v)\bigr). \label{eq-eb}
\end{align}
The EB rule is the posterior mean under the normal model $m_i \sim \mathcal{N}(\mu_i, v_i)$ with $\mu_i \sim \mathcal{N}(\bar{m}_G, \tau^2)$, where $\tau^2$ is estimated by the method of moments. When $\hat{\tau}^2=0$, the rule shrinks every element with positive variance fully to the group mean, and we set $B_i=0$ when $v_i=\hat{\tau}^2=0$. Both rules keep importance non-negative. With $\beta=1$ the importance becomes uniform within each group, which turns the penalty into a plain L2 penalty per group. Study 2 uses $\beta\in\{0.1, 0.25, 0.5, 1\}$, Study 3 uses $\beta=0.25$, and Study 4 uses $\beta\in\{0.25, 0.5\}$.

\subsection{Transient links, and why random links are shrinkage}\label{sec-links}
In Studies 1 and 2 each hidden layer of $H$ units carries a set of transient links (Figure~\ref{fig-model}). At every training step, live links age by one step, links older than $L$ steps die, and $k$ new links form between currently unlinked pairs chosen by a rule. The random rule picks pairs uniformly. The four activity rules score pairs on the current minibatch, where a unit fires on an example when its ReLU output is positive. Co-activity is the fraction of examples on which both units fire, Jaccard similarity is the number of examples on which both fire divided by the number on which either fires, and the correlation and anti-correlation rules use the Pearson correlation of the two activations and its negative, never linking a unit that is silent on the whole minibatch. Each rule links the $k$ highest-scoring unlinked pairs, with ties broken arbitrarily. The importance of the current step then diffuses one step over the live links,
\begin{equation}\label{eq-diffuse}
\iota \leftarrow \iota - a_s \mathcal{L}_s \iota, \qquad a_s = \frac{\alpha}{1 + \max_u d_u(s)},
\end{equation}
where $\mathcal{L}_s = D_s - A_s$ is the Laplacian of the live-link graph at step $s$, $d_u(s)$ is the degree of unit $u$, and $\alpha=1$. The update conserves the total importance of the layer, so links only move protection between units, and the step size keeps it stable for any graph.

\begin{figure}[t]
\centering
\begin{tikzpicture}[x=1cm,y=1cm,
  unit/.style={circle,draw=inkgray,fill=white,line width=0.5pt,inner sep=0pt,minimum size=2.4mm},
  link/.style={draw=vizorange,line width=0.9pt}]
  \def\xs{0.48}
  \begin{scope}
    \node[anchor=west,font=\small] at (0,2.45) {(a) Importance tags};
    \foreach \h [count=\i from 0] in {0.35,1.6,0.55,0.9,2.1,0.3,1.15,0.6} {
      \fill[vizblue] ({0.25+\xs*\i-0.11},0.22) rectangle ++(0.22,{0.85*\h});
      \node[unit] at ({0.25+\xs*\i},0) {};
    }
    \node[font=\scriptsize,text=inkgray] at (1.93,-0.45) {hidden units of one layer};
  \end{scope}
  \begin{scope}[xshift=4.55cm]
    \node[anchor=west,font=\small] at (0,2.45) {(b) Transient random links};
    \foreach \h [count=\i from 0] in {0.35,1.6,0.55,0.9,2.1,0.3,1.15,0.6} {
      \fill[vizblue!35] ({0.25+\xs*\i-0.11},0.22) rectangle ++(0.22,{0.85*\h});
    }
    \foreach \a/\b/\d in {0/3/0.55, 1/6/0.85, 4/5/0.3, 2/7/1.05} {
      \draw[link] ({0.25+\xs*\a},-0.12) .. controls ({0.25+\xs*\a},{-0.12-\d}) and ({0.25+\xs*\b},{-0.12-\d}) .. ({0.25+\xs*\b},-0.12);
    }
    \foreach \i in {0,...,7} { \node[unit] at ({0.25+\xs*\i},0) {}; }
    \node[font=\scriptsize,text=inkgray,align=center] at (1.93,-1.42) {$k$ new links per step, each live $L$ steps};
  \end{scope}
  \begin{scope}[xshift=9.1cm]
    \node[anchor=west,font=\small] at (0,2.45) {(c) Equivalent shrinkage};
    \foreach \h [count=\i from 0] in {0.35,1.6,0.55,0.9,2.1,0.3,1.15,0.6} {
      \draw[inkgray!70,line width=0.4pt] ({0.25+\xs*\i-0.11},0.22) rectangle ++(0.22,{0.85*\h});
      \fill[vizblue] ({0.25+\xs*\i-0.11},0.22) rectangle ++(0.22,{0.85*(0.5*\h+0.5*0.94375)});
      \node[unit] at ({0.25+\xs*\i},0) {};
    }
    \draw[inkgray,dashed,line width=0.6pt] (0.02,{0.22+0.85*0.94375}) -- (3.84,{0.22+0.85*0.94375});
    \node[font=\scriptsize,text=inkgray,anchor=west] at (3.84,{0.22+0.85*0.94375}) {$\bar{\omega}$};
    \node[font=\scriptsize,text=inkgray] at (1.93,-0.45) {$\E[\tilde{\omega}_u] = (1-\beta)\,\omega_u + \beta\,\bar{\omega}$};
  \end{scope}
\end{tikzpicture}
\caption{The consolidation model of Studies 1 and 2. (a) Each hidden unit carries an importance tag, accumulated from squared gradients during a task, that sets how strongly its incoming weights are anchored afterward. (b) Transient links form between random pairs of units, and the importance of each step diffuses over the live links without changing its total. (c) In expectation, the diffusion pulls every tag toward the layer mean by a weight $\beta$ that depends on how many links are alive (Proposition~\ref{prop-links}). This is linear shrinkage, and it needs no graph. Outlines show the tags before shrinkage, drawn here for $\beta=0.5$.}
\Description{Three panels. A row of eight hidden units with bars of different heights above them representing importance tags. In the second panel, curved links connect random pairs of units. In the third panel, the bars are pulled toward a dashed line at the mean height.}
\label{fig-model}
\end{figure}
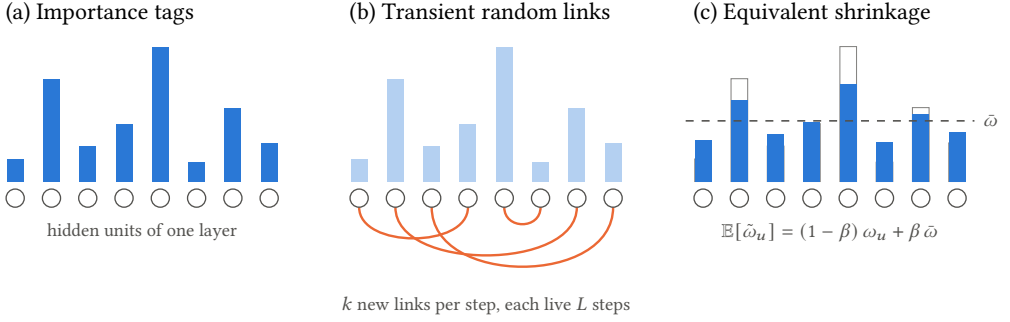

\begin{proposition}[Random links shrink importance toward the layer mean]\label{prop-links}
Suppose the live-link graph at step $s$ is independent of the importance values and its distribution is invariant under any relabeling of the units. Then for any $\iota$ the expected update in Eq.~(\ref{eq-diffuse}) is
\[
\E[\iota'_u] = (1-\beta_s)\,\iota_u + \beta_s\,\bar{\iota}, \qquad \beta_s = \frac{H}{H-1}\,\E\bigl[a_s\, d_u(s)\bigr],
\]
with the same $\beta_s$ for every unit, where $\bar{\iota}$ is the layer mean.
\end{proposition}
\begin{proof}
The step size $a_s$ is a symmetric function of the graph, so by the invariance $a_s A_{uv}(s)$ has the same distribution for every pair $u \neq v$. Hence $\E[a_s A_{uv}(s)] = c_s$ for every pair, and $\E[a_s d_u(s)] = (H-1)c_s$. Because the graph is independent of $\iota$, $\E[\iota'_u] = \iota_u - (H-1)c_s\,\iota_u + c_s\sum_{v\neq u}\iota_v = (1-Hc_s)\,\iota_u + Hc_s\,\bar{\iota}$. Setting $\beta_s = Hc_s$ gives the result.
\end{proof}

Summed over the steps of a task, the expected task importance is approximately the graph-free shrinkage of Eq.~(\ref{eq-fixed}) applied to the summed importance, with $\beta$ the average of $\beta_s$ over the steps of the task. The approximation is exact when $\beta_s$ is constant within the task, which holds closely once the number of live links has reached its steady state, about $L$ steps after training starts. In steady state the link count is $kL$ and the mean degree is $\bar{d}=2kL/H$, so $\beta_s \approx \frac{H}{H-1}\,\bar{d}/(1+d_{\max})$ depends on the lifetime through the mean and the maximum degree. We estimated the effective weight $\beff$ for the six lifetimes of Study 1 by simulating the link process alone with the settings used in training, $H=256$, $k=8$, $\alpha=1$, and 1407 steps per task for 20 tasks (Appendix~\ref{app-sim}). The simulated $\beff$ is 0.23, 0.31, 0.52, 0.65, 0.75, and 0.83 for $L=25$, 50, 200, 500, 1000, and 2000. The simulation reproduces the link density logged during training to the printed precision and the logged coverage to within 0.001. Over one task the realized operator is close to linear shrinkage, with $R^2=0.997$ for $L=25$ and $0.982$ for $L=200$, falling to $0.91$ for $L=2000$, where the graph changes slowly. The activity rules violate both conditions of Proposition~\ref{prop-links}, because their links depend on activations and concentrate on particular units, so they need not shrink toward the layer mean.

\subsection{How heteroscedastic empirical Bayes allocates shrinkage for Fisher estimates}\label{sec-eb}
The EB rule of Eq.~(\ref{eq-eb}) shrinks each estimate by an amount that grows with its sampling variance. For Fisher-type estimates, that variance grows with the square of the importance itself.

\begin{proposition}[EB shrinks the largest Fisher estimates most]\label{prop-eb}
Let the per-example contributions to the EWC estimate for parameter $i$ be $c=g^2$, where $g$ is the derivative of the log-likelihood of a label sampled from the model, and let $\kappa_i$ be the kurtosis of $g$. Then $\E[c]=F_i$, the sampling variance of the mean of $n$ contributions is $v_i=(\kappa_i-1)F_i^2/n$, and for any fixed $\tau^2>0$ the EB weight $B_i=v_i/(v_i+\tau^2)$ increases with $F_i$ when $\kappa_i$ is constant.
\end{proposition}
\begin{proof}
The score has mean zero when labels are drawn from the model, so $\E[g]=0$ and $F_i=\E[g^2]$. Then $\operatorname{Var}(c)=\E[g^4]-(\E[g^2])^2=(\kappa_i-1)F_i^2$, and averaging $n$ independent contributions divides the variance by $n$. The weight $B_i$ increases with $v_i$, and $v_i$ increases with $F_i$.
\end{proof}

For Gaussian $g$, $\kappa_i=3$ and a single contribution has a coefficient of variation of $\sqrt{2}$, so even $n=200$ samples leave each estimate with a relative standard error of 10\%, and heavier-tailed gradients make it larger. Proposition~\ref{prop-eb} implies that EB pulls the most important parameters strongly toward the mean while leaving the least important ones nearly unchanged. That allocation would be right if the noise level of an estimate carried no information about its true value, which is what the normal model assumes. For Fisher estimates the noise level grows with the true value, so the model is misspecified exactly where it matters, and the rule moves protection away from the parameters with the largest estimated importance. A fixed weight instead shrinks every deviation by the same fraction, raising small values as well as lowering large ones. When noise is multiplicative, as here, a standard remedy is to shrink on the log scale. By the delta method, the noise there is approximately additive, with variance near $(\kappa-1)/n$ for every element when $n$ is large, the estimates are positive, and $\kappa$ does not vary across elements. Exponentiating the posterior mean of the log importance gives a power transform, $\tilde{m}_i = m_i^{1-B}\,\bar{m}_{\mathrm{geo}}^{\,B}$, with $\bar{m}_{\mathrm{geo}}$ the geometric mean of the group and a single weight $B$ per group that falls as $n$ grows. We did not test this variant and return to it in Section~\ref{sec-discussion}.

\section{Experimental Design}\label{sec-design}

\subsection{Benchmarks and models}
\paragraph{Permuted MNIST (P-MNIST)} Twenty tasks built from MNIST \citep{lecun1998gradient}, the first with the original pixel order and each of the others with a fixed random permutation of the pixels \citep{goodfellow2013empirical}. All tasks share one 10-way output layer, so the protocol is domain-incremental \citep{vandeven2022three}. The network is a multilayer perceptron with two hidden layers of 256 ReLU units. Studies 1 and 2 train on all 60,000 training images and report test accuracy over a grid of $\lambda$. Studies 3 and 4 hold out 10\% of the training images for validation and select $\lambda$ on validation accuracy only.

\paragraph{Split CIFAR-100 (Study 3)} Ten tasks of ten CIFAR-100 classes each in a fixed class order \citep{krizhevsky2009learning}, with a separate 10-way head per task, a task-incremental protocol. The heads are not regularized.

\paragraph{CIFAR-100 superclass (Study 4)} Five tasks. Task $k$ contains the $k$-th fine class of each of the 20 superclasses, and the label is the superclass, predicted by one shared 20-way head. Unlike the multi-head split, this domain-incremental protocol leaves substantial room for forgetting. All parameters are regularized.

Both CIFAR benchmarks use a convolutional network with four $3\times3$ convolutional layers of 32, 32, 64, and 64 channels with padding 1 and ReLU activations, $2\times2$ max pooling after the second and fourth, and a fully connected layer of 256 ReLU units, without batch normalization. Inputs are standardized per channel with training-set statistics, and we use no data augmentation. The CIFAR networks train for five epochs per task with batches of 64, and the P-MNIST networks for three epochs per task with batches of 128. Unless stated otherwise we use Adam with learning rate $10^{-3}$ \citep{kingma2015adam}, reset at the start of each task, and add the gradient of the penalty to the task gradient. With SGD (learning rate $\eta=0.01$, momentum $\mu=0.9$), an explicit penalty gradient becomes unstable once $2\eta\lambda\Omega_i$ exceeds $2(1+\mu)$, which large importance values can trigger. We therefore apply the penalty as an exact proximal step after each optimizer step, outside the momentum buffer, $\theta_i \leftarrow (\theta_i + c_i\theta^{(t-1)}_i)/(1+c_i)$ with $c_i=2\eta\lambda\Omega_i$ \citep{parikh2014proximal}.

\subsection{The four studies}
\paragraph{Study 1, mechanisms inspired by dendritic nanotubes (10 seeds)} Study 1 compares six arms on P-MNIST. Three have no consolidation, namely a plain MLP, a wider MLP with two hidden layers of 350 units that matches the parameter count of the lateral arm (401,110 against 400,394 parameters), and an MLP with a dense lateral weight matrix $M$ in each hidden layer. The lateral matrix is applied once after the feedforward activation as $h \leftarrow \mathrm{ReLU}(h + hM^{\top})$, with an L1 penalty of $10^{-5}$ on $M$ and magnitude pruning of entries below $2\times10^{-3}$ every 200 steps. The other three arms use unit-level consolidation, either alone (tag only) or with transient links formed at random or by co-activity, with $k=8$ new links per step. We used $\lambda\in\{0.003, 0.01, 0.03\}$ with lifetime $L=200$, and $L\in\{25, 50, 500, 1000, 2000\}$ at $\lambda=0.01$.

\paragraph{Study 2, links versus shrinkage (10 seeds)} The arms are tag only, graph-free shrinkage of the task importance of each hidden layer with $\beta\in\{0.1, 0.25, 0.5, 1\}$ (Eq.~\ref{eq-fixed}, one group per layer, applied before normalization), and transient links under the five rules of Section~\ref{sec-links} with $L\in\{25, 200\}$. All rules form the same number of links per step, so link counts are matched. We used $\lambda\in\{0.001, 0.003, 0.01\}$. We logged the coefficient of variation (CV) of the consolidated importance of each hidden layer at the end of training, averaged over the two hidden layers, together with the mean fraction of units with at least one live link (coverage), the fraction of units linked at least once per task (touched), and the Gini coefficient of link formation across units.

\paragraph{Study 3, shrinkage across estimators (5 seeds)} Study 3 has nine cells, five on P-MNIST (EWC, EWC-unit, SI, and MAS with Adam, and EWC with SGD) and four on Split CIFAR-100 (EWC, SI, and MAS with Adam, and EWC with SGD). EWC and MAS use $n=200$ samples. Each cell has three arms (no shrinkage, fixed $\beta=0.25$, and EB), with one group per parameter tensor (per layer of units for EWC-unit). For each cell, a pilot run with one seed, a shortened sequence of 5 P-MNIST or 4 CIFAR tasks, and no shrinkage selected $\lambda^{*}$ from $\{10^{-4}, 10^{-3}, \dots, 10^{2}\}$ by validation accuracy. Every arm then ran at $\lambda^{*}\times\{1/3, 1, 3\}$, and we report test metrics at the value with the best mean validation accuracy for that arm.

\paragraph{Study 4, variance versus bias (10 seeds)} Study 4 runs EWC and MAS on P-MNIST and on CIFAR-100 superclass with $n\in\{10, 50, 200, 1000\}$ single-example gradients per task, using nested sample sets within a seed, and four arms (no shrinkage, $\beta=0.25$, $\beta=0.5$, and EB). Every arm uses one $\lambda$ per benchmark and method, chosen by a pilot at $n=200$ without shrinkage on the full task sequence from $\{10^{-4}, \dots, 10^{3}\}$. Because $\lambda$ was tuned for the unshrunk arm at $n=200$, this choice favors that arm, at least at that sample size. The two mechanisms make opposite predictions. If shrinkage corrects variance, its gain should fall toward zero as $n$ grows. If it corrects a bias in the importance measure, the gain should persist at large $n$.

\subsection{Metrics and statistics}
Let $R_{t,j}$ be the test accuracy on task $j$ after training on task $t$. We report final average accuracy $\mathrm{ACC}=\frac{1}{T}\sum_j R_{T,j}$, backward transfer $\mathrm{BWT}=\frac{1}{T-1}\sum_{j<T}(R_{T,j}-R_{j,j})$ \citep{lopezpaz2017gradient}, and learning accuracy $\mathrm{LA}=\frac{1}{T}\sum_j R_{j,j}$, a measure of plasticity \citep{riemer2019learning}, all in percent. Within a seed, all arms see the same task permutations, data splits, and class order, so we compare arms by seed-paired differences. We report the mean difference, its standard error (SE), Cohen's $d_z$, a two-sided paired $t$-test, and an exact sign-flip permutation test, and we control the family-wise error with Holm's procedure \citep{holm1979simple}. Significance marks and Holm adjustments use the $t$-test. The sign-flip test serves as a check, and with five seeds its smallest attainable $p$-value is 0.0625. In Study 2 the families are, within each $\lambda$, (A) every shrinkage and link arm against tag only (14 comparisons), (B) each activity rule against random links of the same lifetime (8), and (C) each shrinkage weight against random links of each lifetime (8). Study 3 uses one family of all 27 comparisons, three per cell ($\beta=0.25$ against none, EB against none, and EB against $\beta=0.25$). In Study 4 we regress each seed's gain on $\log_{10} n$, test the mean slope with a one-sample $t$-test and the gain at $n=1000$ with a paired $t$-test, and apply Holm over the 16 primary tests (two fixed weights, slope and gain at $n=1000$, four benchmark and method pairs). Descriptive tables give the mean $\pm$ standard deviation over seeds, computed with the number of seeds as divisor. A non-significant difference is not evidence of equivalence, so where we claim that two arms agree we give a confidence interval.

\subsection{Implementation, compute, and use of AI tools}
All experiments were implemented in PyTorch \citep{paszke2019pytorch} and run on one Apple M5 MacBook Pro with the MPS backend. Configurations shared between the scripts of Studies 1 and 2 reproduce each other exactly. For example, the tag-only arm at $\lambda=0.003$ and the random-link arm with $L=25$ at $\lambda=0.01$ give identical results in both studies. The experiment code, the statistical summaries, the figure scripts, and a first draft of this manuscript were prepared with the assistance of Claude (Anthropic), a large language model, under the author's direction. The author ran all experiments and is responsible for the content. Code and logs will be released with the paper.

\section{Results}\label{sec-results}

\subsection{Study 1, mechanisms inspired by dendritic nanotubes}

\begin{table}[t]
\caption{Study 1 on 20-task P-MNIST (10 seeds). ACC is the mean $\pm$ standard deviation over seeds, and BWT and LA are means, all in percent. The lower panel varies the link lifetime at $\lambda=0.01$, with $\beff$ the shrinkage weight implied by random links (Section~\ref{sec-links}) and density the fraction of possible links that are alive.}
\label{tab-study1}
\small
\begin{tabular}{llccc}
\toprule
Arm & $\lambda$ & ACC & BWT & LA \\
\midrule
MLP (269,322 parameters) & -- & 32.58 $\pm$ 1.42 & $-67.97$ & 97.15 \\
Wider MLP (401,110 parameters) & -- & 35.06 $\pm$ 2.31 & $-65.43$ & 97.21 \\
Lateral links with pruning (400,394 parameters) & -- & 32.34 $\pm$ 1.90 & $-68.20$ & 97.13 \\
\midrule
Tag only & 0.003 & 79.48 $\pm$ 1.42 & $-15.15$ & 93.87 \\
\quad + co-activity links, $L=200$ & 0.003 & 79.94 $\pm$ 1.41 & $-14.65$ & 93.86 \\
\quad + random links, $L=200$ & 0.003 & 82.05 $\pm$ 1.35 & $-12.20$ & 93.64 \\
Tag only & 0.01 & 78.48 $\pm$ 1.42 & $-12.52$ & 90.37 \\
\quad + co-activity links, $L=200$ & 0.01 & 78.32 $\pm$ 1.00 & $-12.52$ & 90.22 \\
\quad + random links, $L=200$ & 0.01 & 80.52 $\pm$ 0.97 & $-9.72$ & 89.75 \\
Tag only & 0.03 & 73.02 $\pm$ 1.37 & $-12.62$ & 85.00 \\
\quad + co-activity links, $L=200$ & 0.03 & 73.17 $\pm$ 1.16 & $-12.20$ & 84.77 \\
\quad + random links, $L=200$ & 0.03 & 75.37 $\pm$ 0.93 & $-8.47$ & 83.41 \\
\bottomrule
\end{tabular}

\vspace{0.8em}
\begin{tabular}{rcccc}
\toprule
Lifetime $L$ & $\beff$ & Random links, ACC & Co-activity links, ACC & Link density \\
\midrule
25 & 0.23 & 81.03 $\pm$ 1.12 & 79.06 $\pm$ 0.88 & 0.6\% \\
50 & 0.31 & 80.90 $\pm$ 0.90 & 78.59 $\pm$ 1.62 & 1.2\% \\
200 & 0.52 & 80.52 $\pm$ 0.97 & 78.32 $\pm$ 1.00 & 4.9\% \\
500 & 0.65 & 80.09 $\pm$ 0.83 & 78.52 $\pm$ 0.89 & 12.1\% \\
1000 & 0.75 & 79.92 $\pm$ 0.86 & 78.00 $\pm$ 1.27 & 24.1\% \\
2000 & 0.83 & 79.48 $\pm$ 0.81 & 77.68 $\pm$ 0.92 & 47.3\% \\
\bottomrule
\end{tabular}
\end{table}

The original idea, dense lateral connections pruned during training, did not help (Table~\ref{tab-study1}). After 20 tasks, pruning had removed all but 1.1\% of the lateral weights, and the network matched the plain MLP (ACC 32.3\% against 32.6\%), while the parameter-matched wider MLP did slightly better (35.1\%). The surviving density depends on the L1 penalty, the pruning threshold, and the length of training, since a 3-task pilot kept 37\%, so it says nothing about how many connections biology uses. Without consolidation all three networks forgot almost everything, with BWT between $-65$ and $-68$ points.

Consolidation dominated everything else, lifting ACC from about 33\% to between 73\% and 80\%. On top of it, links between co-active units changed ACC by less than half a point at every $\lambda$ with $L=200$, and by at most 0.8 points at any lifetime. Random links improved ACC by 2.6, 2.0, and 2.4 points at $\lambda=0.003$, 0.01, and 0.03, through better retention (BWT improved by 2.8 to 4.2 points) at a small cost in plasticity (LA fell by 0.2 to 1.6 points). In the lifetime sweep, random links did best when short-lived and degraded steadily as their lifetime grew, from 81.0\% at $L=25$ to 79.5\% at $L=2000$. The biologically motivated rule never helped, and the control that was meant to be neutral did.

\subsection{Study 2, random links act as shrinkage}

\begin{figure}[t]
\centering
\includegraphics[width=\linewidth]{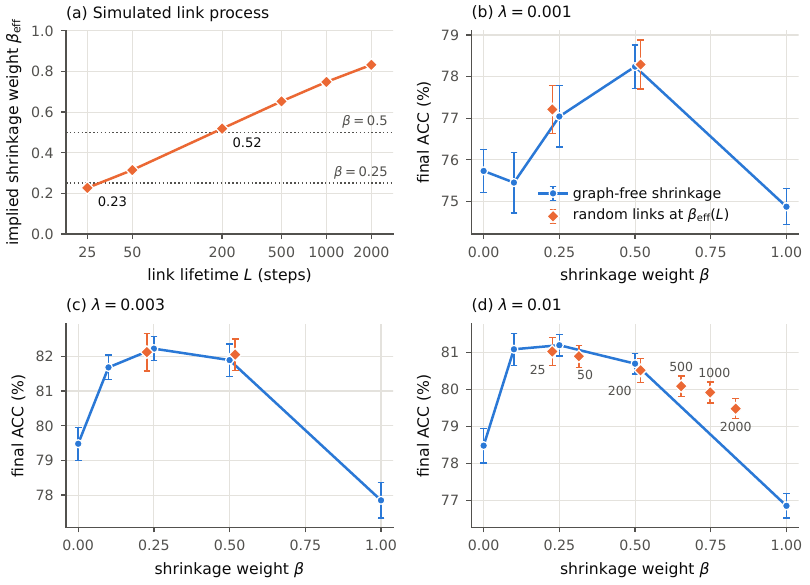}
\caption{Random links behave as graph-free shrinkage at the weight implied by their lifetime (P-MNIST, 20 tasks, 10 seeds). (a) Shrinkage weight implied by the link process, from a simulation of the process alone (Appendix~\ref{app-sim}). (b) to (d) Final accuracy against the shrinkage weight for graph-free shrinkage of hidden-layer importance (blue, where $\beta=0$ is tag only) and for random links placed at their simulated weight (orange), at three penalty strengths $\lambda$. In (d) the labels give the link lifetime, and lifetimes 50, 500, 1000, and 2000 come from Study 1. Error bars show $\pm1$ standard error over seeds. Lines connect measured weights only, so the blue curve was not sampled between 0.5 and 1.}
\Description{Four panels. Panel a shows the implied shrinkage weight rising from 0.23 at lifetime 25 to 0.83 at lifetime 2000. Panels b to d plot final accuracy against shrinkage weight for three penalty strengths. The graph-free shrinkage curve peaks between beta 0.25 and 0.5 and falls at beta 1. Random-link points sit on the curve at weights 0.23 and 0.52. In panel d, longer lifetimes sit at larger weights and lower accuracy.}
\label{fig-links}
\end{figure}

\begin{table}[t]
\caption{Study 2 at $\lambda=0.003$ (P-MNIST, 10 seeds). ACC is the mean $\pm$ standard deviation in percent. Differences are seed-paired, in points, with the standard error in parentheses. Graph-free arms are compared with random links of the matched lifetime ($\beta=0.25$ with $L=25$, $\beta=0.5$ with $L=200$), and link rules with random links of the same lifetime. CV is the coefficient of variation of the consolidated importance. \sig{} marks Holm-adjusted $p<0.05$ within the family. Appendix Table~\ref{tab-study2-all} gives the other two values of $\lambda$.}
\label{tab-study2}
\small
\setlength{\tabcolsep}{3.5pt}
\begin{tabular}{lcrrcccc}
\toprule
Arm & ACC & $\Delta$ tag only & $\Delta$ random & CV & Coverage & Touched & Gini \\
\midrule
Tag only & 79.48 $\pm$ 1.42 & -- & -- & 0.543 & -- & -- & -- \\
Shrinkage, $\beta=0.1$ & 81.68 $\pm$ 1.07 & $+2.20$ (0.44)\sig & -- & 0.450 & -- & -- & -- \\
Shrinkage, $\beta=0.25$ & 82.22 $\pm$ 1.03 & $+2.74$ (0.43)\sig & $+0.10$ (0.30) & 0.407 & -- & -- & -- \\
Shrinkage, $\beta=0.5$ & 81.89 $\pm$ 1.40 & $+2.41$ (0.50)\sig & $-0.17$ (0.21) & 0.320 & -- & -- & -- \\
Shrinkage, $\beta=1$ & 77.85 $\pm$ 1.53 & $-1.63$ (0.61) & -- & 0.000 & -- & -- & -- \\
\midrule
Random, $L=25$ & 82.12 $\pm$ 1.63 & $+2.64$ (0.41)\sig & -- & 0.413 & 0.792 & 1.000 & 0.013 \\
Random, $L=200$ & 82.05 $\pm$ 1.35 & $+2.58$ (0.43)\sig & -- & 0.313 & 0.999 & 1.000 & 0.013 \\
Co-activity, $L=25$ & 79.73 $\pm$ 1.60 & $+0.25$ (0.32) & $-2.38$ (0.33)\sig & 0.497 & 0.159 & 0.493 & 0.659 \\
Co-activity, $L=200$ & 79.94 $\pm$ 1.41 & $+0.46$ (0.37) & $-2.11$ (0.46)\sig & 0.450 & 0.453 & 0.721 & 0.502 \\
Jaccard, $L=25$ & 79.90 $\pm$ 0.72 & $+0.42$ (0.39) & $-2.22$ (0.50)\sig & 0.497 & 0.214 & 0.699 & 0.625 \\
Jaccard, $L=200$ & 80.42 $\pm$ 1.20 & $+0.94$ (0.33) & $-1.63$ (0.36)\sig & 0.450 & 0.638 & 0.909 & 0.456 \\
Correlation, $L=25$ & 80.19 $\pm$ 1.83 & $+0.71$ (0.41) & $-1.93$ (0.45)\sig & 0.522 & 0.555 & 0.959 & 0.344 \\
Correlation, $L=200$ & 80.59 $\pm$ 1.40 & $+1.12$ (0.40) & $-1.46$ (0.41)\sig & 0.445 & 0.962 & 0.986 & 0.201 \\
Anti-correlation, $L=25$ & 80.18 $\pm$ 0.99 & $+0.70$ (0.28) & $-1.94$ (0.36)\sig & 0.518 & 0.441 & 0.860 & 0.482 \\
Anti-correlation, $L=200$ & 79.55 $\pm$ 1.72 & $+0.07$ (0.49) & $-2.50$ (0.33)\sig & 0.453 & 0.871 & 0.968 & 0.302 \\
\bottomrule
\end{tabular}
\end{table}

Figure~\ref{fig-links} and Table~\ref{tab-study2} report Study 2. Four observations point to shrinkage as the mechanism.

First, graph-free shrinkage reproduces the gain of random links. At $\lambda=0.003$, shrinkage with $\beta=0.25$ improved ACC over tag only by $2.74\pm0.43$ points (mean $\pm$ SE, Holm $p=0.002$), and random links with $L=25$ improved it by $2.64\pm0.41$ (Holm $p=0.002$). Shrinkage with $\beta=0.5$ improved it by $2.41\pm0.50$, and links with $L=200$ by $2.58\pm0.43$. Paired directly, the matched arms differ by at most 0.18 points at any $\lambda$. The 95\% confidence intervals of these differences lie within $\pm0.8$ points at $\lambda=0.003$ and $0.01$, a third or less of the gain over tag only, and within $\pm1.3$ points at $\lambda=0.001$, where the paired differences were noisier.

Second, the pairing is predicted rather than fitted. The simulation of the link process alone gives $\beff=0.23$ for $L=25$ and $0.52$ for $L=200$ (Figure~\ref{fig-links}a), close to the graph-free weights that match each lifetime. The consolidated importance agrees. At $\lambda=0.001$, 0.003, and 0.01 its CV was 0.394, 0.413, and 0.454 for links with $L=25$ against 0.389, 0.407, and 0.444 for $\beta=0.25$, and 0.295, 0.313, and 0.333 for links with $L=200$ against 0.302, 0.320, and 0.342 for $\beta=0.5$. In the lifetime sweep of Study 1 (Figure~\ref{fig-links}d), lifetimes of 25 and 200 steps land within 0.2 points of the matched shrinkage arms. Longer lifetimes imply larger weights and gave lower accuracy, as shrinkage predicts. They lie 0.6 to 1.3 points above a straight line between the measured arms at $\beta=0.5$ and $\beta=1$, but we did not sample the shrinkage curve in that range, so the sweep cannot test the match there.

Third, partial shrinkage beats both extremes. Full shrinkage ($\beta=1$), which anchors every unit of a hidden layer equally, was worse than tag only at every $\lambda$, by 0.9 to 1.6 points, and worse than random links by 2.3 to 4.3 points (Holm $p\le0.001$ in all six comparisons). Per-unit importance carries real information, and partial shrinkage keeps it while damping its extremes.

Fourth, every activity rule did worse than random links. All 24 comparisons with random links of the same lifetime and $\lambda$ were negative, by 0.4 to 2.7 points, and 21 were Holm-significant. The diagnostics explain part of the gap. Raw co-firing concentrated links on a few units. At $L=25$ it touched only 49\% of the units in a task, with a Gini coefficient of 0.66 against 100\% and 0.01 for random links, and it shrank less (CV 0.497 against 0.413 at $\lambda=0.003$). Correlation links covered almost every unit at $L=200$ (coverage 0.96) and still trailed random links by 1.5 points. At matched dispersion they also trailed graph-free shrinkage, for example 80.6\% for correlation links with CV 0.445 against 81.7\% for $\beta=0.1$ with CV 0.450. For the rules that link similar units, one explanation is that averaging importance among units that already behave alike pulls each tag toward a local mean rather than the layer mean. Anti-correlation links also trailed random links, by 1.8 to 2.6 points, so the more general reading is that every activity rule concentrates links on a structured subset of pairs and therefore does not implement the uniform shrinkage of Proposition~\ref{prop-links}.

\subsection{Study 3, fixed shrinkage across estimators}

\begin{table}[t]
\caption{Study 3 (5 seeds). Final accuracy in percent, mean $\pm$ standard deviation, at each arm's validation-selected $\lambda$, with seed-paired differences from no shrinkage in points (SE in parentheses). $\Delta$LA is the change in learning accuracy under $\beta=0.25$, and $\bar{B}$ is the mean EB weight. \unadj{} marks unadjusted $p<0.05$. No difference survives Holm correction over the 27 comparisons. Appendix Table~\ref{tab-lambda} lists the selected $\lambda$ values.}
\label{tab-study3}
\small
\setlength{\tabcolsep}{5pt}
\begin{tabular}{llcccrrrc}
\toprule
Method & Opt. & None & $\beta=0.25$ & EB & $\Delta$, $\beta=0.25$ & $\Delta$LA & $\Delta$, EB & $\bar{B}$ \\
\midrule
\multicolumn{9}{l}{\emph{P-MNIST}} \\
EWC & Adam & 82.1 $\pm$ 1.2 & 82.6 $\pm$ 0.5 & 81.5 $\pm$ 0.5 & $+0.47$ (0.65) & $-1.11$ & $-0.61$ (0.83) & 0.25 \\
EWC-unit & Adam & 80.8 $\pm$ 1.0 & 80.6 $\pm$ 1.4 & 81.6 $\pm$ 0.8 & $-0.20$ (0.43) & $-0.09$ & $+0.86$ (0.28)\unadj & 0.37 \\
SI & Adam & 82.3 $\pm$ 1.2 & 83.5 $\pm$ 1.2 & 81.9 $\pm$ 0.8 & $+1.24$ (0.30)\unadj & $-0.23$ & $-0.41$ (0.36) & 0.13 \\
MAS & Adam & 81.7 $\pm$ 0.5 & 83.6 $\pm$ 0.7 & 82.0 $\pm$ 0.8 & $+1.92$ (0.34)\unadj & $-0.11$ & $+0.29$ (0.53) & 0.03 \\
EWC & SGD & 82.8 $\pm$ 0.9 & 82.5 $\pm$ 1.0 & 82.6 $\pm$ 0.7 & $-0.29$ (0.55) & $+2.48$ & $-0.21$ (0.37) & 0.29 \\
\midrule
\multicolumn{9}{l}{\emph{Split CIFAR-100}} \\
EWC & Adam & 57.1 $\pm$ 1.2 & 54.8 $\pm$ 0.7 & 57.2 $\pm$ 1.1 & $-2.31$ (0.57)\unadj & $-3.53$ & $+0.08$ (0.15) & 0.05 \\
SI & Adam & 55.0 $\pm$ 0.9 & 54.2 $\pm$ 0.5 & 55.0 $\pm$ 0.9 & $-0.79$ (0.36) & $-1.51$ & $+0.01$ (0.06) & 0.19 \\
MAS & Adam & 56.5 $\pm$ 1.0 & 54.2 $\pm$ 0.4 & 56.6 $\pm$ 1.0 & $-2.31$ (0.49)\unadj & $-3.32$ & $+0.08$ (0.47) & 0.03 \\
EWC & SGD & 59.6 $\pm$ 1.1 & 57.9 $\pm$ 1.0 & 59.2 $\pm$ 1.0 & $-1.65$ (0.24)\unadj & $-1.77$ & $-0.35$ (0.16) & 0.07 \\
\bottomrule
\end{tabular}
\end{table}

On P-MNIST, $\beta=0.25$ improved SI by $1.24\pm0.30$ points and MAS by $1.92\pm0.34$ (Table~\ref{tab-study3}; $d_z=1.86$ and $2.49$, unadjusted $p=0.014$ and $0.005$). The gains came almost entirely from retention, with BWT improving by 1.5 and 2.1 points while learning accuracy barely moved ($-0.2$ and $-0.1$; Appendix Table~\ref{tab-study3-bwt}). The changes for EWC ($+0.5$), EWC-unit ($-0.2$), and EWC with SGD ($-0.3$) were within noise. On Split CIFAR-100, $\beta=0.25$ reduced accuracy in all four cells, by 0.8 to 2.3 points. The losses came from plasticity, with LA falling by 1.5 to 3.5 points, while retention improved slightly (BWT up by 0.1 to 1.4 points). The mean EB weights ranged from 0.03 to 0.37, and EB changed accuracy by $-0.6$ to $+0.9$ points, reaching unadjusted significance only for EWC-unit. On P-MNIST it trailed fixed shrinkage for SI and MAS by 1.6 points each (unadjusted $p=0.004$).

No comparison in Study 3 survives Holm correction over all 27 tests, and the smallest adjusted $p$ is 0.06. With five seeds, Study 3 is a screen, and it motivated Study 4, which tests the MAS result again with ten seeds. Twenty of the 27 selected values of $\lambda$ also sat at an edge of the three-point bracket, so the bracket was too narrow to locate every optimum (Appendix Table~\ref{tab-lambda}).

\subsection{Study 4, variance for EWC and bias for MAS}

\begin{figure}[t]
\centering
\includegraphics[width=\linewidth]{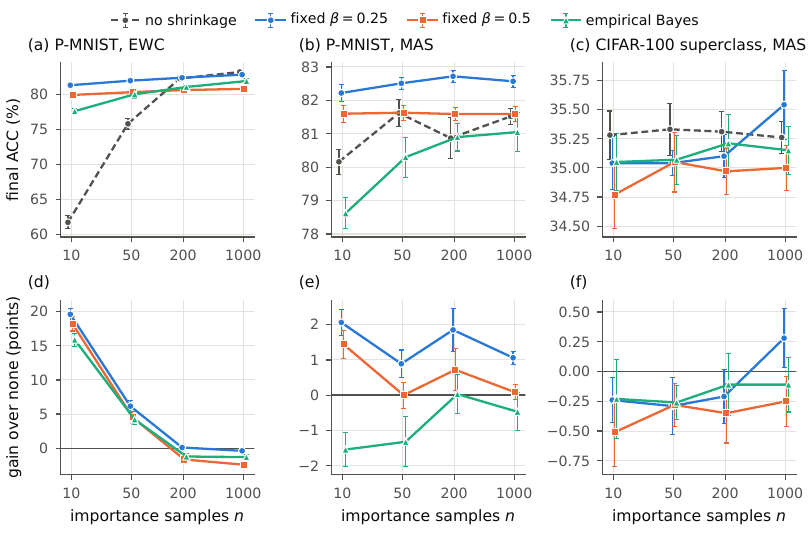}
\caption{Varying the number of importance samples $n$ separates variance from bias (Study 4, 10 seeds). (a) to (c) Final accuracy of each arm, with error bars of $\pm1$ standard error over seeds. (d) to (f) Seed-paired gain over no shrinkage, mean $\pm1$ standard error. For EWC on P-MNIST the gain vanishes as $n$ grows, the signature of variance. For MAS on P-MNIST a fixed weight of 0.25 helps at every $n$, the signature of bias. Points are offset horizontally for legibility. The CIFAR-100 superclass EWC cells are pending a rerun.}
\Description{Six panels in two rows. Top row shows final accuracy against the number of importance samples for EWC on P-MNIST, MAS on P-MNIST, and MAS on CIFAR-100 superclass, for no shrinkage, fixed beta 0.25, fixed beta 0.5, and empirical Bayes. For EWC, the unshrunk arm rises from 62 to 83 percent while the shrunk arms stay between 77 and 83 percent. Bottom row shows the gain over no shrinkage. For EWC the gain falls from between 16 and 20 points at 10 samples to between minus 2.4 and minus 0.4 points at 1000 samples. For MAS on P-MNIST, beta 0.25 stays positive at every n. For CIFAR-100 all gains are near zero.}
\label{fig-nsweep}
\end{figure}

\begin{table}[t]
\caption{Study 4 (10 seeds). The upper panel gives the seed-paired gain in final accuracy over no shrinkage, in points (SE in parentheses), and the seed-level slope of the gain against $\log_{10} n$, in points per decade. \sig{} marks Holm-adjusted $p<0.05$ over the 16 primary tests, which are the slopes and the gains at $n=1000$ for the two fixed weights; EB rows are descriptive. The lower panel gives the accuracy of the unshrunk arm with its slope against $\log_{10} n$, the coefficient of variation (CV) of its raw importance estimates averaged over tensors and tasks, and the mean weight $\bar{B}$ assigned by the EB arm.}
\label{tab-study4}
\small
\setlength{\tabcolsep}{5pt}
\begin{tabular}{lrrrrr}
\toprule
 & $n=10$ & $n=50$ & $n=200$ & $n=1000$ & Slope per decade \\
\midrule
\multicolumn{6}{l}{\emph{P-MNIST, EWC}} \\
\quad $\beta=0.25$ & $+19.60$ (0.87) & $+6.16$ (0.87) & $+0.12$ (0.37) & $-0.37$ (0.28) & $-9.99$ (0.42)\sig \\
\quad $\beta=0.5$ & $+18.20$ (1.04) & $+4.52$ (0.90) & $-1.64$ (0.37) & $-2.41$ (0.30)\sig & $-10.30$ (0.46)\sig \\
\quad EB & $+15.89$ (0.96) & $+4.22$ (0.80) & $-1.20$ (0.39) & $-1.29$ (0.40) & $-8.62$ (0.62) \\
\multicolumn{6}{l}{\emph{P-MNIST, MAS}} \\
\quad $\beta=0.25$ & $+2.06$ (0.36) & $+0.89$ (0.39) & $+1.85$ (0.61) & $+1.06$ (0.18)\sig & $-0.33$ (0.17) \\
\quad $\beta=0.5$ & $+1.44$ (0.39) & $+0.00$ (0.37) & $+0.72$ (0.59) & $+0.08$ (0.21) & $-0.53$ (0.14)\sig \\
\quad EB & $-1.54$ (0.47) & $-1.32$ (0.71) & $+0.03$ (0.55) & $-0.47$ (0.54) & $+0.68$ (0.38) \\
\multicolumn{6}{l}{\emph{CIFAR-100 superclass, MAS}} \\
\quad $\beta=0.25$ & $-0.24$ (0.19) & $-0.29$ (0.24) & $-0.21$ (0.23) & $+0.28$ (0.25) & $+0.25$ (0.16) \\
\quad $\beta=0.5$ & $-0.51$ (0.29) & $-0.28$ (0.18) & $-0.35$ (0.25) & $-0.25$ (0.21) & $+0.11$ (0.17) \\
\quad EB & $-0.23$ (0.33) & $-0.26$ (0.14) & $-0.11$ (0.26) & $-0.11$ (0.23) & $+0.08$ (0.24) \\
\multicolumn{6}{l}{\emph{CIFAR-100 superclass, EWC}\quad \pending{pending rerun}} \\
\midrule
\multicolumn{6}{l}{\emph{P-MNIST, EWC}} \\
\quad ACC, none & 61.7 & 75.8 & 82.2 & 83.2 & $+10.74$ (0.45) \\
\quad CV, none & 4.47 & 2.81 & 1.70 & 1.41 &  \\
\quad Mean EB weight & 0.49 & 0.30 & 0.22 & 0.10 &  \\
\multicolumn{6}{l}{\emph{P-MNIST, MAS}} \\
\quad ACC, none & 80.2 & 81.6 & 80.9 & 81.5 & $+0.52$ (0.17) \\
\quad CV, none & 1.00 & 0.84 & 0.79 & 0.78 &  \\
\quad Mean EB weight & 0.24 & 0.08 & 0.03 & 0.006 &  \\
\multicolumn{6}{l}{\emph{CIFAR-100 superclass, MAS}} \\
\quad ACC, none & 35.3 & 35.3 & 35.3 & 35.3 & $-0.01$ (0.12) \\
\quad CV, none & 1.17 & 0.95 & 0.90 & 0.89 &  \\
\quad Mean EB weight & 0.24 & 0.09 & 0.03 & 0.007 &  \\
\bottomrule
\end{tabular}
\end{table}

For EWC on P-MNIST the number of samples mattered enormously (Figure~\ref{fig-nsweep}a and d, Table~\ref{tab-study4}). Without shrinkage, ACC rose from 61.7\% with $n=10$ to 83.2\% with $n=1000$, by $10.7\pm0.5$ points per tenfold increase in $n$. The gain from $\beta=0.25$ fell in step, from 19.6 points at $n=10$ to 6.2, 0.1, and $-0.4$ points at $n=50$, 200, and 1000, a slope of $-10.0\pm0.4$ points per decade (Holm $p<10^{-4}$). With $\beta=0.5$ the gain became a significant loss of 2.4 points at $n=1000$. The raw importance estimates tell the same story. Their CV fell from 4.47 to 1.41 as $n$ grew, and the EB weight fell from 0.49 to 0.10. At $n=10$, EB recovered 15.9 of the 19.6 points that $\beta=0.25$ gained, but it lost 1.2 to 1.3 points at $n\ge200$. This is the signature of variance. Shrinkage substitutes for samples, and once the samples suffice it removes signal.

For MAS on P-MNIST, more samples helped the unshrunk arm only slightly ($0.5\pm0.2$ points per decade), and $\beta=0.25$ helped at every $n$, by 2.1, 0.9, 1.9, and 1.1 points. Its gain had a slope of $-0.33\pm0.17$ points per decade (95\% CI $-0.71$ to $0.05$, unadjusted $p=0.08$), and its gain of $1.06\pm0.18$ points at $n=1000$ was Holm-significant. With $\beta=0.5$ the gain was smaller and declined with $n$ (slope $-0.5\pm0.1$ points per decade, Holm-adjusted $p<0.05$), so part of the benefit at small $n$ is a variance effect. The raw importance of MAS was far less dispersed than that of EWC, with a CV of 1.00 at $n=10$ and 0.78 at $n=1000$, and the EB weight fell from 0.24 at $n=10$ to 0.006 at $n=1000$. EB did not help, and it hurt at small $n$. At $n=10$ its mean weight matched the fixed weight of 0.25, yet EB lost 1.5 points where $\beta=0.25$ gained 2.1, so where shrinkage falls matters as much as how much there is. The gain of $\beta=0.25$ at $n=1000$ shows a component that more samples do not remove. On this benchmark, the converged MAS importance protects old tasks less well than a tempered version of itself, which is what we mean by bias.

On CIFAR-100 superclass, MAS showed no detectable effect of shrinkage at any $n$, with all gains within $\pm0.6$ points and a slope of $0.25\pm0.16$ points per decade, and no detectable effect of $n$ ($-0.01\pm0.12$ points per decade). Every arm forgot heavily, with BWT near $-31$ points, and reached about 35\% accuracy, which suggests that the selected penalty was weak on this benchmark. A $\lambda=0$ reference, part of the rerun described next, will show whether regularization helps there at all.

The EWC cells of this benchmark are pending. In the first run, the importance accumulator fixed its numerical unit from the first sample. When that sample had a near-zero gradient, the unit became very large and later sums overflowed single-precision arithmetic, which produced non-finite dispersion values and EB weights in these cells. We replaced it with an accumulator that rescales as larger contributions arrive and are rerunning the cells \pending{results to be added}. In every other cell of Studies 3 and 4, which used the earlier accumulator, all EB weights were finite and no run diverged, so the overflow did not affect those results, and the unshrunk and fixed-weight arms do not use the variance estimates at all. The P-MNIST results and the CIFAR-100 MAS results do not depend on the pending cells. We also checked that the Holm-adjusted significance of every result marked in Table~\ref{tab-study4} holds whatever values the four pending tests take, the closest case being the slope for MAS with $\beta=0.5$ at 0.047.

\section{Discussion}\label{sec-discussion}

\subsection{Two reasons to shrink}
The four studies point to one simple intervention with two distinct justifications. When importance comes from few samples, as with EWC at $n=10$ or 50, the estimates are noisy and heavy-tailed, the penalty concentrates on the parameters that happened to receive large values, and shrinkage spreads protection back out. Its benefit then disappears as samples accumulate, and too much shrinkage hurts. When the importance measure is precise but miscalibrated for retention, as MAS appears to be on P-MNIST, a small fixed weight helps at every sample size. Fixed shrinkage mixes the importance-weighted penalty with a uniform one within each group, $\lambda\sum_i[(1-\beta)\Omega_i + \beta\bar{\Omega}_{G(i)}](\theta_i-\theta^{*}_i)^2$, where $\bar{\Omega}_{G(i)}$ is the mean importance of the group that contains parameter $i$. It therefore adds an L2 floor that limits the cumulative drift of parameters that every task rates as unimportant. That reading fits the finding that MAS approximates the square root of the Fisher information \citep{benzing2022unifying}, a profile that is already flattened relative to EWC, but the floor is a hypothesis we did not test directly.

The distinction suggests a practical check. Compare the unshrunk method at two sample sizes. If it improves with more samples, the problem is variance, and the remedy is more samples or shrinkage that adapts to $n$. If it does not, a fixed shrinkage weight chosen on validation data is worth trying as a second hyperparameter beside $\lambda$.

\subsection{Why the empirical Bayes rule did not deliver}
EB moved in the right direction for the variance effect, shrinking hard when $n$ was small and little when it was large, yet it lost to a fixed weight at small $n$ and to no shrinkage at large $n$. Proposition~\ref{prop-eb} suggests one reason. For Fisher-type estimates the sampling noise is multiplicative, so the heteroscedastic rule assigns its largest weights to the largest estimates. It pulls the most important parameters toward the mean and leaves the rest almost untouched, which removes protection from the parameters with the largest estimated importance. A log-scale rule would shrink every estimate by the same factor in log space, a power transform of the importance whose exponent approaches one as $n$ grows. We expect such a rule to keep the adaptivity of EB to $n$ without concentrating shrinkage on the largest estimates, and testing it is the obvious next experiment. For MAS, where much of the benefit at large $n$ does not come from noise, a variance-driven rule cannot recover it, and the EB weights of Study 4 were indeed near zero for $n\ge50$. At $n=10$, where EB shrank as much on average as the fixed weight and still lost accuracy, the allocation again appears to be the problem.

\subsection{Shrinkage trades plasticity for stability}
Because every task has a fixed protection budget, shrinkage moves protection from high-importance to low-importance parameters, and the latter are the ones new tasks would otherwise use freely. Shrinkage lowered learning accuracy in most cells of Studies 3 and 4, a form of the intransigence that consolidation penalties are known to cause \citep{chaudhry2018riemannian}. It therefore pays when forgetting is the bottleneck, as on P-MNIST with a shared head, where BWT without shrinkage ranged from $-8$ to $-13$ points in Study 3, and it costs accuracy when forgetting is not, as on the multi-head Split CIFAR-100, where BWT ranged from $-0.9$ to $-1.6$ points. The CIFAR-100 superclass benchmark has plenty of forgetting but showed no detectable effect for MAS, possibly because the selected penalty was weak there.

\subsection{What remains of the nanotube idea}
Our results give no support to the hypothesis that activity-dependent exchange of plasticity signals improves consolidation. The biologically motivated rules never beat random links, and random links appear to help because, in expectation, they implement a statistical correction that needs no graph at all. The results do establish a weaker fact that may still be useful. Any mass-conserving exchange of plasticity signals between randomly paired units acts as shrinkage of noisy importance, with a strength set by how many links are alive and how long they last. Whether biological nanotubes form at random with respect to plasticity signals is an empirical question that these experiments cannot answer.

The same logic matters for applied continual learning with noisy, subject-specific data. In EEG-based emotion recognition, importance heuristics degrade under noise and covariate shift \citep{peire2026affect}, and fixed or log-scale shrinkage of importance is a cheap intervention to test in that setting.

\subsection{Limitations}
\begin{itemize}
\item The networks and benchmarks are small, a two-layer MLP and a four-layer CNN on MNIST and CIFAR-100. We have not tested larger models, longer task sequences, or other domains.
\item Study 3 used five seeds, and none of its comparisons survives correction for 27 tests. Twenty of its 27 selected values of $\lambda$ lie at the edge of the searched bracket, sometimes in opposite directions for the shrunk and unshrunk arms (for EWC on Split CIFAR-100, the upper edge without shrinkage and the lower edge with it), so its effect sizes are conditional on the bracket.
\item Study 4 used one $\lambda$ per benchmark and method, chosen without shrinkage at $n=200$. This favors the unshrunk arm at that sample size, may understate the benefit of shrinkage, and ties the CIFAR-100 superclass results to a penalty that may be weak. It also leaves open whether the persistent MAS gain reflects a bias of the importance measure or an interaction between $\beta$ and $\lambda$, which tuning $\lambda$ for each arm would resolve.
\item Fixed shrinkage is equivalent to adding a groupwise L2 penalty toward the anchor to the method's own penalty. We did not compare it with a separately tuned combination of the method and an L2 penalty.
\item The EB rule uses a normal model with tensor-level groups, and the sampling variance of SI ignores autocorrelation. We tested one shrinkage target, the group mean.
\item The CIFAR-100 superclass EWC cells and the $\lambda=0$ reference are pending a rerun.
\item Our links abstract dendritic nanotubes and do not model their physiology.
\end{itemize}

\section{Conclusion}\label{sec-conclusion}
The question in our title has a short answer. Links that follow activity, the feature we borrowed from dendritic nanotubes, did not help our networks, while random links did, because in expectation they shrink noisy importance estimates. A biologically inspired mechanism for sharing plasticity signals between neurons thus turned out, in its only working form, to be a shrinkage estimator in disguise. Taking that seriously led to a simple intervention for regularization-based continual learning, shrinking importance estimates toward their mean, and to a test that separates two reasons it helps. On Permuted MNIST, for EWC with few samples it reduces variance, and its benefit vanishes as samples grow. For MAS it also corrects something that more samples do not remove, which we read as bias. The heteroscedastic empirical Bayes rule we tried shrinks the largest Fisher-type estimates most, and a log-scale rule is the natural next step. Importance estimates are estimates, and treating them as such adds almost no computation.

\bibliographystyle{ACM-Reference-Format}
\bibliography{references}

\appendix

\section{Hyperparameters}\label{app-hyper}

\begin{table}[h]
\caption{Settings of the four studies.}
\label{tab-hyper}
\footnotesize
\setlength{\tabcolsep}{4pt}
\begin{tabularx}{\linewidth}{@{}l>{\raggedright\arraybackslash}X>{\raggedright\arraybackslash}X>{\raggedright\arraybackslash}X>{\raggedright\arraybackslash}X@{}}
\toprule
 & Study 1 & Study 2 & Study 3 & Study 4 \\
\midrule
Benchmarks & P-MNIST, 20 tasks & P-MNIST, 20 tasks & P-MNIST, 20 tasks; Split CIFAR-100, 10 tasks & P-MNIST, 20 tasks; CIFAR-100 superclass, 5 tasks \\
Model & MLP $2\times256$, wider MLP $2\times350$ & MLP $2\times256$ & MLP $2\times256$; CNN & MLP $2\times256$; CNN \\
Epochs, batch & 3, 128 & 3, 128 & 3, 128; 5, 64 & 3, 128; 5, 64 \\
Optimizer & Adam $10^{-3}$ & Adam $10^{-3}$ & Adam $10^{-3}$; SGD 0.01 with momentum 0.9 and proximal penalty & Adam $10^{-3}$ \\
Importance & unit-level squared gradients & unit-level squared gradients & EWC, EWC-unit, MAS ($n=200$); SI ($\xi=0.1$) & EWC, MAS ($n\in\{10,50,200,1000\}$) \\
Shrinkage & none & $\beta\in\{0.1,0.25,0.5,1\}$, per hidden layer & none, $\beta=0.25$, EB; per tensor & none, $\beta\in\{0.25,0.5\}$, EB; per tensor \\
Links & $k=8$, $\alpha=1$, $L\in\{25$, 50, 200, 500, 1000, $2000\}$; random and co-activity & $k=8$, $\alpha=1$, $L\in\{25,200\}$; five rules & -- & -- \\
$\lambda$ & $\{0.003,0.01,0.03\}$ & $\{0.001,0.003,0.01\}$ & pilot, then $\lambda^{*}\times\{1/3,1,3\}$ & pilot at $n=200$ \pending{values} \\
Validation & none & none & 10\% of training data & 10\% of training data \\
Seeds & 10 & 10 & 5 & 10 \\
\bottomrule
\end{tabularx}
\end{table}

\begin{table}[h]
\caption{Study 3, pilot $\lambda^{*}$ and the value selected for each arm on validation accuracy from $\lambda^{*}\times\{1/3,1,3\}$. \unadj{} marks a selected value at the edge of the bracket.}
\label{tab-lambda}
\small
\begin{tabular}{lllcccc}
\toprule
Benchmark & Method & Opt. & Pilot $\lambda^{*}$ & None & $\beta=0.25$ & EB \\
\midrule
P-MNIST & EWC & Adam & 0.01 & 0.01 & 0.01 & 0.00333\unadj \\
P-MNIST & EWC-unit & Adam & 0.01 & 0.00333\unadj & 0.00333\unadj & 0.01 \\
P-MNIST & SI & Adam & 0.001 & 0.003\unadj & 0.003\unadj & 0.003\unadj \\
P-MNIST & MAS & Adam & 0.001 & 0.003\unadj & 0.003\unadj & 0.003\unadj \\
P-MNIST & EWC & SGD & 0.1 & 0.3\unadj & 0.1 & 0.1 \\
Split CIFAR-100 & EWC & Adam & 10 & 30\unadj & 3.33\unadj & 30\unadj \\
Split CIFAR-100 & SI & Adam & 1 & 3\unadj & 3\unadj & 3\unadj \\
Split CIFAR-100 & MAS & Adam & 100 & 33.3\unadj & 100 & 100 \\
Split CIFAR-100 & EWC & SGD & 1 & 0.333\unadj & 0.333\unadj & 0.333\unadj \\
\bottomrule
\end{tabular}
\end{table}

\section{Simulation of the link process}\label{app-sim}
The simulation replicates the random rule of the training scripts without any training. A layer has $H=256$ units. At each step, links formed $L$ steps earlier die, and $k=8$ new links form between pairs drawn uniformly from the currently unlinked pairs, which is equivalent to taking the top $k$ of independent uniform scores over those pairs. The step size is $a_s=1/(1+\max_u d_u(s))$, and we record $\frac{H}{H-1}\,a_s\,\bar{d}(s)$, the shrinkage weight of Proposition~\ref{prop-links}, for each of 1407 steps per task over 20 tasks. To check how well linear shrinkage describes the realized operator, we diffused one fixed log-normal importance vector at every step of the eleventh task, accumulated the result, and regressed it on the vector. Table~\ref{tab-sim} gives the results. The logged density and coverage come from Studies 1 and 2.

\begin{table}[h]
\caption{The link process of the random rule, simulated without training. $\beff$ is averaged over 20 tasks; the first-task value includes the period in which links accumulate. The realized weight and $R^2$ come from a linear fit of the accumulated operator over one task.}
\label{tab-sim}
\small
\setlength{\tabcolsep}{4pt}
\begin{tabular}{rccccccc}
\toprule
$L$ & $\beff$ & $\beff$, first task & Realized weight & $R^2$ & Coverage & Density & Logged density, coverage \\
\midrule
25 & 0.227 & 0.226 & 0.224 & 0.997 & 0.793 & 0.0061 & 0.006, 0.792 \\
50 & 0.315 & 0.307 & 0.312 & 0.995 & 0.957 & 0.0122 & 0.012 \\
200 & 0.518 & 0.508 & 0.510 & 0.982 & 0.999 & 0.0488 & 0.049, 0.999 \\
500 & 0.652 & 0.608 & 0.635 & 0.957 & 0.999 & 0.1215 & 0.121 \\
1000 & 0.748 & 0.648 & 0.706 & 0.939 & 0.999 & 0.2407 & 0.241 \\
2000 & 0.832 & 0.652 & 0.828 & 0.906 & 0.999 & 0.4728 & 0.473 \\
\bottomrule
\end{tabular}
\end{table}

\clearpage
\section{Additional results}\label{app-results}

\begin{table}[h]
\caption{Study 2 at $\lambda=0.001$ and $\lambda=0.01$ (P-MNIST, 10 seeds). ACC is the mean in percent, and differences are seed-paired, in points, with SE in parentheses. Graph-free arms are compared with random links of the matched lifetime, and link rules with random links of the same lifetime. \sig{} marks Holm-adjusted $p<0.05$ within the family.}
\label{tab-study2-all}
\footnotesize
\setlength{\tabcolsep}{4pt}
\begin{tabular}{lrrrrrrrr}
\toprule
 & \multicolumn{4}{c}{$\lambda=0.001$} & \multicolumn{4}{c}{$\lambda=0.01$} \\
\cmidrule(lr){2-5}\cmidrule(lr){6-9}
Arm & ACC & $\Delta$ tag only & $\Delta$ random & CV & ACC & $\Delta$ tag only & $\Delta$ random & CV \\
\midrule
Tag only & 75.73 & -- & -- & 0.460 & 78.48 & -- & -- & 0.733 \\
$\beta=0.1$ & 75.45 & $-0.28$ (0.71) & -- & 0.423 & 81.09 & $+2.61$ (0.46)\sig & -- & 0.507 \\
$\beta=0.25$ & 77.04 & $+1.31$ (0.73) & $-0.17$ (0.47) & 0.389 & 81.20 & $+2.72$ (0.50)\sig & $+0.16$ (0.13) & 0.444 \\
$\beta=0.5$ & 78.24 & $+2.51$ (0.52)\sig & $-0.05$ (0.31) & 0.302 & 80.70 & $+2.22$ (0.45)\sig & $+0.18$ (0.14) & 0.342 \\
$\beta=1$ & 74.87 & $-0.86$ (0.44) & -- & 0.000 & 76.85 & $-1.63$ (0.47) & -- & 0.000 \\
Random, $L=25$ & 77.21 & $+1.48$ (0.43) & -- & 0.394 & 81.03 & $+2.55$ (0.54)\sig & -- & 0.454 \\
Random, $L=200$ & 78.29 & $+2.57$ (0.40)\sig & -- & 0.295 & 80.52 & $+2.04$ (0.46)\sig & -- & 0.333 \\
Co-activity, $L=25$ & 75.99 & $+0.26$ (0.60) & $-1.22$ (0.53) & 0.438 & 79.06 & $+0.58$ (0.39) & $-1.97$ (0.26)\sig & 0.644 \\
Co-activity, $L=200$ & 76.28 & $+0.55$ (0.65) & $-2.02$ (0.52)\sig & 0.415 & 78.32 & $-0.16$ (0.44) & $-2.20$ (0.30)\sig & 0.551 \\
Jaccard, $L=25$ & 76.29 & $+0.56$ (0.50) & $-0.92$ (0.48) & 0.439 & 78.96 & $+0.48$ (0.44) & $-2.07$ (0.47)\sig & 0.638 \\
Jaccard, $L=200$ & 76.29 & $+0.56$ (0.55) & $-2.01$ (0.57)\sig & 0.403 & 78.50 & $+0.02$ (0.46) & $-2.02$ (0.49)\sig & 0.547 \\
Correlation, $L=25$ & 76.82 & $+1.09$ (0.52) & $-0.39$ (0.36) & 0.434 & 78.37 & $-0.11$ (0.32) & $-2.66$ (0.45)\sig & 0.671 \\
Correlation, $L=200$ & 75.99 & $+0.26$ (0.71) & $-2.31$ (0.53)\sig & 0.384 & 78.35 & $-0.13$ (0.47) & $-2.17$ (0.44)\sig & 0.573 \\
Anti-correlation, $L=25$ & 75.39 & $-0.33$ (0.53) & $-1.81$ (0.41)\sig & 0.429 & 78.44 & $-0.04$ (0.32) & $-2.59$ (0.41)\sig & 0.698 \\
Anti-correlation, $L=200$ & 76.16 & $+0.43$ (0.51) & $-2.13$ (0.43)\sig & 0.400 & 78.43 & $-0.05$ (0.42) & $-2.09$ (0.40)\sig & 0.592 \\
\bottomrule
\end{tabular}
\end{table}

\begin{table}[h]
\caption{Study 3 retention and plasticity at each arm's selected $\lambda$ (5 seeds, means in percent).}
\label{tab-study3-bwt}
\small
\begin{tabular}{lllrrrrrr}
\toprule
 & & & \multicolumn{3}{c}{BWT} & \multicolumn{3}{c}{LA} \\
\cmidrule(lr){4-6}\cmidrule(lr){7-9}
Benchmark & Method & Opt. & None & $\beta=0.25$ & EB & None & $\beta=0.25$ & EB \\
\midrule
P-MNIST & EWC & Adam & $-10.4$ & $-8.7$ & $-12.9$ & 92.0 & 90.9 & 93.8 \\
 & EWC-unit & Adam & $-13.5$ & $-13.6$ & $-8.9$ & 93.6 & 93.5 & 90.1 \\
 & SI & Adam & $-11.8$ & $-10.3$ & $-12.2$ & 93.5 & 93.3 & 93.5 \\
 & MAS & Adam & $-12.6$ & $-10.5$ & $-12.3$ & 93.7 & 93.6 & 93.7 \\
 & EWC & SGD & $-7.9$ & $-10.8$ & $-10.6$ & 90.3 & 92.8 & 92.6 \\
\midrule
Split CIFAR-100 & EWC & Adam & $-1.3$ & 0.0 & $-1.2$ & 58.3 & 54.8 & 58.3 \\
 & SI & Adam & $-0.9$ & $-0.1$ & $-0.8$ & 55.8 & 54.3 & 55.7 \\
 & MAS & Adam & $-1.1$ & 0.0 & $-1.4$ & 57.5 & 54.2 & 57.8 \\
 & EWC & SGD & $-1.6$ & $-1.5$ & $-1.4$ & 61.0 & 59.2 & 60.5 \\
\bottomrule
\end{tabular}
\end{table}

\begin{table}[h]
\caption{Study 4 retention and plasticity at the smallest and largest sample sizes (10 seeds, means in percent).}
\label{tab-study4-bwt}
\small
\begin{tabular}{llrrrrrr}
\toprule
 & & \multicolumn{3}{c}{$n=10$} & \multicolumn{3}{c}{$n=1000$} \\
\cmidrule(lr){3-5}\cmidrule(lr){6-8}
Setting & Arm & ACC & BWT & LA & ACC & BWT & LA \\
\midrule
\multirow{4}{*}{P-MNIST, EWC} & None & 61.7 & $-34.6$ & 94.5 & 83.2 & $-8.5$ & 91.3 \\
 & $\beta=0.25$ & 81.3 & $-11.4$ & 92.1 & 82.8 & $-8.2$ & 90.6 \\
 & $\beta=0.5$ & 79.9 & $-11.7$ & 91.0 & 80.8 & $-9.9$ & 90.2 \\
 & EB & 77.6 & $-14.1$ & 91.0 & 81.9 & $-9.4$ & 90.8 \\
\midrule
\multirow{4}{*}{P-MNIST, MAS} & None & 80.2 & $-10.9$ & 90.5 & 81.5 & $-9.2$ & 90.3 \\
 & $\beta=0.25$ & 82.2 & $-8.1$ & 90.0 & 82.6 & $-7.6$ & 89.8 \\
 & $\beta=0.5$ & 81.6 & $-8.5$ & 89.7 & 81.6 & $-8.5$ & 89.6 \\
 & EB & 78.6 & $-12.3$ & 90.3 & 81.1 & $-9.7$ & 90.3 \\
\midrule
\multirow{4}{*}{\shortstack[l]{CIFAR-100\\superclass, MAS}} & None & 35.3 & $-31.0$ & 60.1 & 35.3 & $-31.0$ & 60.1 \\
 & $\beta=0.25$ & 35.0 & $-31.5$ & 60.3 & 35.5 & $-31.0$ & 60.3 \\
 & $\beta=0.5$ & 34.8 & $-31.8$ & 60.2 & 35.0 & $-31.5$ & 60.2 \\
 & EB & 35.1 & $-31.2$ & 60.0 & 35.2 & $-31.2$ & 60.1 \\
\bottomrule
\end{tabular}
\end{table}

\end{document}